\documentclass[sigconf]{acmart}
\pdfoutput=1 

\usepackage{mathtools}
\usepackage{xspace}
\newcommand{\newowa}{proxCSL\xspace}
\usepackage{amsmath}
\usepackage{amsfonts}
\usepackage{amsthm}

\usepackage{microtype}
\usepackage{graphicx}
\usepackage{subfigure}

\usepackage{algorithm}
\usepackage[noend]{algorithmic}
\usepackage{dsfont}
\usepackage{array}
\usepackage{ccaption}
\usepackage{tikz}
\usepackage{enumitem}
\usepackage{placeins}
\usepackage{subcaption}
\usepackage{adjustbox}

\usepackage{todonotes}

\newtheorem{thm}{Theorem}
\newtheorem{lem}{Lemma}

\newtheorem{defn}{Definition}

\usepackage{pgfplots}
\DeclareUnicodeCharacter{2212}{−}
\usepgfplotslibrary{groupplots,dateplot}
\usetikzlibrary{patterns,shapes.arrows}
\pgfplotsset{compat=newest}

\pgfplotsset{every axis/.append style={
                    label style={font=\footnotesize},
                    tick label style={font=\footnotesize}  
                    }}
                    
\usetikzlibrary{arrows}
\usepackage{url}
\usepackage{booktabs}

\DeclareMathOperator*{\argmin}{arg\,min}

\let\L\undefined %
\let\O\undefined %
\newcommand{\L}{\mathcal{L}}

\newcommand{\X}{\mathcal{X}}
\newcommand{\Y}{\mathcal{Y}}
\newcommand{\I}{\mathcal{I}}

\newcommand{\wowa}{\hat{w}_{\mathrm{owa}}}
\newcommand{\proxL}{\tilde{\mathcal{L}}^{(t)}}
\newcommand{\O}{\mathcal{O}}

\AtBeginDocument{%
  \providecommand\BibTeX{{%
    \normalfont B\kern-0.5em{\scshape i\kern-0.25em b}\kern-0.8em\TeX}}}

\copyrightyear{2024}
\acmYear{2024}
\setcopyright{licensedusgovmixed}
\acmConference[KDD '24]{Proceedings of
the 30th ACM SIGKDD Conference on Knowledge Discovery and Data
Mining}{August 25--29, 2024}{Barcelona, Spain}
\acmBooktitle{Proceedings of the 30th ACM SIGKDD Conference on Knowledge
Discovery and Data Mining (KDD '24), August 25--29, 2024, Barcelona, Spain}
\acmDOI{10.1145/3637528.3672038}
\acmISBN{979-8-4007-0490-1/24/08}

\begin{document}

\title{High-Dimensional Distributed Sparse Classification with
Scalable Communication-Efficient Global Updates}

\author{Fred Lu}
\affiliation{%
  \institution{Booz Allen Hamilton}
  \city{McLean}
  \country{USA} 
  }
\affiliation{
  \institution{University of Maryland, Baltimore County}
  \city{Baltimore}
  \country{USA}
}
\email{lu\_fred@bah.com}

\author{Ryan R. Curtin}
\affiliation{%
  \institution{Booz Allen Hamilton}
  \city{McLean}
\country{USA} 
  }
\email{curtin\_ryan@bah.com}

\author{Edward Raff}
\affiliation{%
  \institution{Booz Allen Hamilton}
  \city{McLean}
  \country{USA} 
  }
\affiliation{
  \institution{University of Maryland, Baltimore County}
  \city{Baltimore}
  \country{USA}
}
\email{raff\_edward@bah.com}

\author{Francis Ferraro}
\affiliation{%
  \institution{University of Maryland, Baltimore County}
  \city{Baltimore}
\country{USA} 
  }
\email{ferraro@umbc.edu}

\author{James Holt}
\affiliation{%
  \institution{Laboratory for Physical Sciences}
  \city{College Park}
  \country{USA} 
  }
\email{holt@lps.umd.edu}

\renewcommand{\shortauthors}{Fred Lu, Ryan R. Curtin, Edward Raff, Francis Ferraro, \& James Holt}

\begin{abstract}
As the size of datasets used in statistical learning continues to grow,
distributed training of models has attracted increasing attention.
These methods partition the data and exploit parallelism to reduce memory and runtime,
but suffer increasingly from communication costs as the data size or the number of iterations grows.
Recent work on linear models has shown that a surrogate likelihood can be optimized locally to iteratively improve on an initial solution in a communication-efficient manner.
However, existing versions of these methods experience multiple shortcomings as the data size becomes massive,
including diverging updates and efficiently handling sparsity.
In this work we develop solutions to these problems which enable us to learn a communication-efficient distributed logistic regression model
even beyond millions of features.
In our experiments we demonstrate a large improvement in accuracy over distributed algorithms with only a few distributed update steps needed,
and similar or faster runtimes. Our code is available at \url{https://github.com/FutureComputing4AI/ProxCSL}.
\end{abstract}



\maketitle

\section{Introduction}
\label{intro}

Over the past decade, the size of datasets used in statistical and machine learning has increased dramatically.
When the number of samples and covariates in a dataset becomes sufficiently large,
even the training of linear models over the entire dataset becomes computationally challenging.
This has sparked a flurry of interest in distributed training and inference methods.
By splitting the data over multiple machines,
local training processes can be run in parallel to save memory and runtime.
Multiple works have studied the distributed training of logistic regression models~\cite{zhuang2015distributed, gopal2013distributed, lin2014large},
partitioning the dataset along samples or features and iteratively communicating gradients or gradient surrogates.

However, when many iterations are needed for convergence,
the communication cost of iterative distributed algorithms start to dominate.
For example, when the number of features $d$ of the dataset is massive,
second-order optimization methods which need to communicate $\mathcal{O}(d^2)$ information become impractical.
Yet first-order methods,
while communicating only $\mathcal{O}(d)$ information at a time,
have slower convergence guarantees so may be even more inefficient due to the extra rounds of communication needed.

To alleviate this bottleneck, recent works have proposed methods to train distributed linear models with relatively little communication.
Such methods are first initialized with a distributed one-shot estimator across a dataset which is partitioned across multiple nodes. 
To do this, the linear model objective is solved locally on each machine,
and the results are transmitted to a central processor which merges the local solutions \cite{chen2014split}.

From this initial estimate, such methods then obtain gradient information from all the partitions. 
The local machine can then solve a modified objective which takes into account the global gradient.
This process can be iterated leading to convergence to the full data solution.
Such an approach can be interpreted as an update step which only communicates first-order,
but uses local second- (or higher-) order information to achieve better convergence rates.
Many recent papers have studied variants of this underlying approach,
including~\cite{shamir2014communication,jordan2018communication,wang2017efficient,fan2023communication}.
These all share the underlying update objective,
referred to as the \textit{communication-efficient surrogate likelihood} (CSL).

Theory has been progressively developed for these methods showing a favorable rate of convergence under certain conditions.
These can include, for example, nearness of the local estimated gradient to the global gradient,
and sufficient strong convexity of the Hessian.
In practice, as the number of features or partitions of the dataset grows,
these conditions often fail to hold,
leading to diverging solutions.
For high-dimensional data, having sparse and interpretable model weights is also of interest.
Yet introducing sparsity further complicates the problem,
as few standard solvers can solve the CSL objective with $L_1$ regularization in an efficient manner.
These are important limitations for the practical use of such methods for training large-scale academic or industry models.
In these real-world scenarios, the data dimensionality can be exceedingly large,
while also being sparse,
leading to high correlations among features and poor conditioning of the objective.
Existing experimental results from these prior works have not assessed their methods on data of this size,
as they have only tested on moderately-sized data with low dimensionality $d$ relative to the partition sample size $n$.

In our work, we first show that a standard implementation of CSL methods fails to effectively learn sparse logistic regression models on these high-dimensional datasets.
We next develop an effective solver for the $L_1$-regularized CSL objective which scales efficiently beyond millions of features,
and prove that it converges to the correct solution.
Experimentally this approach attains higher accuracies than other methods when the solution is highly sparse.
However, at low regularizations when the solution is only moderately sparse, the solution to the CSL objective often diverges sharply, leading to poor update performance.
To address this, we develop an adaptive damping method to stabilize the CSL objective for high-dimensional solutions.
Using this technique, we demonstrate across multiple single-node and multi-node distributed experiments that our method successfully performs communication-efficient updates to improve accuracy across a wide range of sparsity settings.

\section{Background and Related Work}
\label{sec:related_work}

\subsection{Sparse logistic regression}

We assume a dataset $\mathcal{D} = (\mathcal{X}, \mathcal{Y})$ where $\mathcal{X} = \{ x_1, \ldots, x_N \}$
consists of $N$ samples in $d$ dimensions
(that is, each $x_i \in \mathbb{R}^d$).
The samples are labeled by $\mathcal{Y} = \{ y_1, \ldots, y_N \}$, where each $y_i \in \{0, 1\}$.

The standard approach to obtaining a sparse logistic regression solution is to use L1 regularization (also known as the LASSO penalty) \cite{hastie2015statistical}.
The objective of this problem is
\begin{equation}
    \label{eqn:obj}
    w^* \coloneqq \argmin_{w \in \mathbb{R}^d} \frac{1}{N} \sum_{i = 1}^{N} \ell(y_i, x_i^\top w) + \lambda \| w \|_1,
\end{equation}
\noindent where $\ell(y, z) = \log(1 + e^z) - yz$.
By setting $\lambda$ appropriately, the solution $\hat{w}$ can show good performance while having few nonzeros compared to the dimensionality of the data: $\| w \|_0 \ll d$.

While many algorithms exist to solve the problem \cite{boyd2011distributed,lee2014proximal,goldstein2014field},
iterated proximal Newton steps using coordinate descent to solve a quadratic approximation of the objective are known to be especially efficient,
seeing wide use in popular machine and statistical learning software \cite{fan2008liblinear,friedman2010regularization}.
In particular, the \textit{newGLMNET} algorithm in the LIBLINEAR library is perhaps the most commonly used solver \cite{yuan2011improved}. 

\subsection{Distributed estimation}
As datasets grow in size, the memory and computational limit of a single machine leads practitioners to seek distributed methods. 
One approach uses exact iterative methods starting from scratch,
which are usually based on adding parallelism to standard single-core algorithms. 
For example, LIBLINEAR-MP is a modified version of newGLMNET which uses multiple cores for certain inner linear algebra operations \cite{zhuang2018naive}. 
Alternatively, stochastic gradient methods allow data to be partitioned across machines or to be sampled at each iteration, reducing the memory requirement \cite{zinkevich2010parallelized}.

For even larger datasets, partitioning data across multiple nodes becomes necessary.
Distributed methods which handle splitting by samples include distributed Newton methods \cite{shamir2014communication,zhuang2015distributed} and ADMM \cite{boyd2011distributed},
while splitting over features has been proposed in block coordinate descent implementations such as \cite{trofimov2015distributed,richtarik2016distributed}.

An important limitation to all these approaches is the communication overhead involved in transmitting information (e.g. gradients) across nodes.
Because of this, when data becomes especially large or many iterations are needed for convergence,
communication costs start to dominate.
As the size of data further increases, one-shot or few-shot methods become increasingly attractive by eliminating most of the communication overhead to obtain an approximate solution.

\subsection{One-shot estimation}
\label{sec:oneshot}

Suppose the $N$ samples of $\mathcal{D}$ are partitioned across $p$ nodes or machines,
and let $\{\mathcal{D}_1, \ldots, \mathcal{D}_p\}$ denote the samples on each partition, with each $\mathcal{D}_i = (\mathcal{X}_i, \mathcal{Y}_i)$.
For simplicity we assume each node has $n$ samples, but this is not required. We define global and local objective functions respectively as
\begin{equation}
    \mathcal{L}(w) \coloneqq \frac{1}{N}\sum_{i=1}^N \ell(y_i, x_i^\top w) + \lambda \lVert w \rVert_1 \label{eq:global_obj}
\end{equation}
and 
\begin{equation}
    \mathcal{L}_k(w) \coloneqq \frac{1}{n}\sum_{(x_i, y_i) \in \mathcal{D}_k} \ell(y_i, x_i^\top w) + \lambda \lVert w \rVert_1 \label{eq:local_obj}
\end{equation}

Each machine first locally solves (\ref{eq:local_obj}) to obtain weight vector $\hat{w}_{(k)}$.
Then the weights are communicated to a central merge node, where they need to be efficiently merged. 
In the \textit{naive average}, a uniform average is taken:
\begin{equation}
\hat{w}_{\textrm{na}} \coloneqq \frac{1}{p} \sum_{k = 1}^p \hat{w}_{(k)}. \label{eq:naive_avg}
\end{equation}

While naive averaging is asymptomptically optimal for nearly unbiased linear models \cite{zhang2012communication},
in high-dimensional models higher-order loss terms cause increasing approximation error \cite{rosenblatt2016optimality}.
Other merge strategies such as improved weightings or debiasing steps have been explored, as in \cite{chen2014split,liu2014distributed,zhang2012communication,lee2017communication}.
Many of these do not scale well to high-dimensional data, as discussed in \cite{izbicki2020distributed}.

A more recent work proposed computing the merge weights by performing a second-stage empirical risk minimization over a subsample of the original data \cite{izbicki2020distributed}. Let $\mathcal{D}_{\mathrm{sub}}$ represent the subsampled data, which is usually just the local data $\mathcal{D}_1$ stored on the merge node. The local solutions are combined column-wise into a matrix $\hat{W}\in \mathbb{R}^{d \times p}$. Then the merge weighting $v\in \mathbb{R}^p$ is estimated as

\begin{equation}
\hat{v} \coloneqq \argmin_{v \in \mathbb{R}^p} \frac{1}{|\mathcal{D}_{\mathrm{sub}}|}\sum_{(x_i, y_i) \in \mathcal{D}_{\mathrm{sub}}} \ell(y_i, x_i^\top \hat{W} v) + \lambda_{\mathrm{cv}} \| v \|_2.
\label{eqn:cv_obj}
\end{equation}
where $\lambda_{\mathrm{cv}}$ is chosen using cross-validation.
Then the final solution is $\hat{w}_{\mathrm{owa}} \coloneqq \hat{W} \hat{v}$.

This method is fast and scalable to the largest datasets,
while generally improving over prior one-shot estimators \cite{izbicki2020distributed}.

\subsection{Communication-efficient updates}
\label{sec:more_communication}

Non-interactive estimators are approximate and degrade in performance as the local sample size $n$ decreases,
which happens when the number of partitions $p$ grows.
As a result, our interest is in update procedures which can be iterated to improve the initial estimator,
ideally approaching the full data solution.

Similar frameworks for iterative global updates have been proposed in algorithms, such as DANE \cite{shamir2014communication}, ILEA \cite{jordan2018communication}, and EDSL \cite{wang2017efficient}. For simplicity, we will adopt the term \textit{communication-efficient surrogate likelihood} (CSL) from \cite{jordan2018communication} and refer to works using this framework as CSL-type methods.

Broadly, they propose solving locally the objective
\begin{equation}
    \tilde{\mathcal{L}}_k(w) \coloneqq \mathcal{L}_k(w) + \Big(  \nabla \mathcal{L}(\hat{w}) - \nabla \mathcal{L}_k(\hat{w}) \Big)^\top w \label{eq:csl}
\end{equation}

This is motivated as optimizing a Taylor expansion of the local objective (\ref{eq:local_obj}),
where the local gradient $\nabla \mathcal{L}_k(\hat{w})$ is replaced with the global gradient $\nabla \mathcal{L}(\hat{w})$.
The affine term can be viewed as a first-order correction for the gradient direction.
To give further intuition, the higher-order derivatives beyond the Hessian are disregarded if we take a quadratic approximation of the local objective
\begin{equation} q_{\mathcal{L}_k}(w) = \nabla \mathcal{L}_k(\hat{w})^\top \delta   + \frac{1}{2} \delta^\top H_{(k)}(\hat{w}) \delta \end{equation}
where $\delta \coloneqq w - \hat{w}$ and ignoring constant terms.
If so, then optimizing CSL simplifies to finding
\begin{equation}  \argmin_\delta \nabla \mathcal{L}(\hat{w})^\top \delta + \frac{1}{2} \delta^\top H_{(k)}(\hat{w}) \delta  \end{equation}
which is equivalent to a quasi-Newton step
using global gradients and local Hessian.

Depending on the method, the local objective can either be updated only on the main node (CSL) or on all nodes simultaneously (DANE).
The latter requires another round of communication and averaging per iteration.
The result of the update $\hat{w}^{(1)}$ then becomes the starting point for the next update iteration. 

Other strategies for communication-efficient updates have been proposed. The ACOWA approach~\cite{lu2024optimizing} also seeks to reduce the impact of many processors $p$, by performing two rounds of computation and attempting to adjust the loss/increase information sharing in those two rounds over its predecessor OWA~\cite{izbicki2020distributed}. Single-machine-only algorithms based on lock-free parallelism as a similar issue, where the Hogwild algorithm~\cite{,recht2011hogwild}  would regularly diverge, and a lock-free approach SAUS~\cite{raff2018linear,JMLR:v18:16-131} attempted to reduce divergences with careful design. In contrast, our approach is iterative but does not need many rounds of iteration in practice and supports both distributed and single-machine parallelism.
\subsection{Challenges for scaling CSL-like methods}

Our work aims to solve practical and theoretical concerns when applying the CSL framework to update models on massive datasets and highly distributed systems.
We first identify challenges in the existing methods.

\textbf{Sparsity.} When the dimensionality of the data is enormous, model sparsity is a desirable property.
For the case of $L_1$-regularized linear models, the local loss term $\mathcal{L}_k(w)$ includes a $\lambda \lVert w \rVert_1$ term.
To efficiently optimize this objective requires techniques specialized for handling the non-differentiable 1-norm.
While this setting has been discussed or studied in \cite{jordan2018communication, wang2017efficient,fan2023communication},
none of the prior works proposed or specified what solver to use. 

Thus a practitioner must apply an out-of-the-box solver or implement their own.
Due to the size of data involved, first-order or dual solvers such as proximal gradient descent or ADMM would likely converge slowly.
In our experiments we instead use OWL-QN, a sparse variant of L-BFGS \cite{andrew2007scalable}, which we believe to be the fastest standard solver.
Because it uses approximate second-order information, it has faster convergence than the alternatives while still scaling up to high-dimensional data.
We find that this baseline implementation is adequate on smaller datasets, where relatively few features are impactful, but often fails to converge or return sparse solutions on larger data. 

We note that this may not have been an issue for prior work because (1) their experiments were limited to lower-dimensional or synthetic datasets, and (2) they did not measure the actual sparsity of their models at any $\lambda$.

To address this issue, we develop an efficient and scalable solver based on iterative proximal quasi-Newton steps, which we detail in Section~\ref{sec:method}.
Using this solver, our method \newowa successfully converges to the true objective as well as the right sparsity (Fig.~\ref{fig:convergence}).

\textbf{Divergence of the CSL objective.} As the data size increases, so often will the number of distributed partitions to facilitate the use of larger computing systems. 
If $d$ or $p$ grow faster than $N$, this often results in a local sample size $n$ which is comparable to $d$ or smaller. 
This causes the curvature of the local objective $\mathcal{L}_k$ to decrease.
In particular the local Hessian may become poorly conditioned or not positive definite at all. 
Furthermore, the affine term of the CSL objective grows with $\lVert \nabla \mathcal{L}(\hat{w}) - \nabla \mathcal{L}_k(\hat{w}) \rVert$,
which may also increase when $N$ and $n$ diverge.
In fact, much of the existing convergence theory relies on upper bounds of the above term.

Under such conditions, the optimum of the CSL objective $\hat{w}^{(1)}$ may be enormous, leading to a diverging update
\begin{equation}\lVert \hat{w}^{(1)} - w^* \rVert \gg \lVert \hat{w} - w^* \rVert \end{equation}

To lessen this effect, we apply damping to the local second-order information, which increases convexity and improves conditioning.
This is equivalent to adding an additional proximal regularization term $\frac{\alpha}{2} \lVert w - \hat{w} \rVert_2^2$.
We note that this term has also been proposed in \cite{shamir2014communication,fan2023communication}.
However, using our solver, we are able to propose an adaptive criterion for increasing $\alpha$ when divergence occurs. Thus $\alpha$ is used and adjusted only when necessary.
In contrast, prior methods need to tune $\alpha$ for each optimization,
an expensive task for large datasets. Altogether the full objective for finding $\hat{w}^{(t+1)}$ is
\begin{equation}
    \proxL(w) \coloneqq \mathcal{L}_k(w) + \Big(  \nabla \mathcal{L}(\hat{w}^{(t)}) - \nabla \mathcal{L}_k(\hat{w}^{(t)}) \Big)^\top w \\
    + \frac{\alpha}{2} \lVert w - \hat{w}^{(t)} \rVert_2^2 \label{eq:full_csl}
\end{equation}

Refer to Fig~\ref{fig:divergence} for demonstrative examples of the CSL objective diverging and the effect of $\alpha$ in fixing it.

\begin{figure}[t]
    \subfigure[{\it amazon7}, 128 partitions, $\lambda=0.001$.]{
        \begin{tikzpicture}[scale=0.8]

\definecolor{darkgray176}{RGB}{176,176,176}
\definecolor{darkorange25512714}{RGB}{255,127,14}
\definecolor{steelblue31119180}{RGB}{31,119,180}

\begin{axis}[
width=0.32\textwidth,
height=0.27\textwidth,
legend cell align={left},
legend style={fill opacity=0.8, draw opacity=1, text opacity=1, draw=none},
tick align=outside,
tick pos=left,
x grid style={darkgray176},
xlabel={Iterations},
xmin=-0.5, xmax=10.5,
xtick style={color=black},
y grid style={darkgray176},
ylabel={Objective error (\%)},
ymin=-0.0288308091609998, ymax=0.569228980936544,
ytick style={color=black},
major tick length=0.1,
]
\path [fill=darkorange25512714, fill opacity=0.2]
(axis cs:0,0.541502138446953)
--(axis cs:0,0.374574672985446)
--(axis cs:1,0.0230251581476573)
--(axis cs:2,0.00739977016650742)
--(axis cs:3,0.0011568678741418)
--(axis cs:4,0.000368235095538498)
--(axis cs:5,0.000534681314206345)
--(axis cs:6,0.0005170076176228)
--(axis cs:7,-0.00164627324747508)
--(axis cs:8,0.000476047477802161)
--(axis cs:9,0.000239472030875886)
--(axis cs:10,0.00066746160540338)
--(axis cs:10,0.00306816416215274)
--(axis cs:10,0.00306816416215274)
--(axis cs:9,0.0018907054839171)
--(axis cs:8,0.00317696134211534)
--(axis cs:7,0.00740489778694148)
--(axis cs:6,0.0076977964462379)
--(axis cs:5,0.0033795756914876)
--(axis cs:4,0.00451063059609937)
--(axis cs:3,0.00332677620311623)
--(axis cs:2,0.0109612882235786)
--(axis cs:1,0.0396922695625135)
--(axis cs:0,0.541502138446953)
--cycle;

\path [fill=steelblue31119180, fill opacity=0.2]
(axis cs:0,0.542044445023019)
--(axis cs:0,0.375058253816028)
--(axis cs:1,0.0221775899810661)
--(axis cs:2,0.00362676916810539)
--(axis cs:3,0.0013091910364534)
--(axis cs:4,0.000392442111276676)
--(axis cs:5,0.000295250805779172)
--(axis cs:6,0.000210642852603842)
--(axis cs:7,0.00022430459214971)
--(axis cs:8,0.000187846626028752)
--(axis cs:9,0.000219490709938619)
--(axis cs:10,0.000189789940979605)
--(axis cs:10,0.00362979604717084)
--(axis cs:10,0.00362979604717084)
--(axis cs:9,0.00137758764183202)
--(axis cs:8,0.0036317393621217)
--(axis cs:7,0.00138617581571971)
--(axis cs:6,0.00364914930384297)
--(axis cs:5,0.00143584810697915)
--(axis cs:4,0.00384260761593585)
--(axis cs:3,0.00219991398540886)
--(axis cs:2,0.00960552755341685)
--(axis cs:1,0.0431842376126508)
--(axis cs:0,0.542044445023019)
--cycle;

\addplot [semithick, darkorange25512714, mark=*, mark size=1, mark options={solid}]
table {%
0 0.4580384057162
1 0.0313587138550854
2 0.00918052919504299
3 0.00224182203862902
4 0.00243943284581894
5 0.00195712850284697
6 0.00410740203193035
7 0.0028793122697332
8 0.00182650440995875
9 0.00106508875739649
10 0.00186781288377806
};
\addlegendentry{\small sCSL}
\addplot [semithick, steelblue31119180, mark=*, mark size=1, mark options={solid}]
table {%
0 0.458551349419524
1 0.0326809137968585
2 0.00661614836076112
3 0.00175455251093113
4 0.00211752486360626
5 0.000865549456379159
6 0.0019298960782234
7 0.000805240203934713
8 0.00190979299407522
9 0.000798539175885319
10 0.00190979299407522
};
\addlegendentry{\small proxCSL}
\end{axis}

\end{tikzpicture}
        \begin{tikzpicture}[scale=0.8]

\definecolor{darkgray176}{RGB}{176,176,176}
\definecolor{darkorange25512714}{RGB}{255,127,14}
\definecolor{steelblue31119180}{RGB}{31,119,180}

\begin{axis}[
width=0.32\textwidth,
height=0.27\textwidth,
legend cell align={left},
legend style={fill opacity=0.8, draw opacity=1, text opacity=1, draw=none},
tick align=outside,
tick pos=left,
x grid style={darkgray176},
xlabel={Iterations},
xmin=-0.5, xmax=10.5,
xtick style={color=black},
y grid style={darkgray176},
ylabel={Number of non-zeros},
ymin=50, ymax=710,
ytick style={color=black},
major tick length=0.1,
]
\path [fill=darkorange25512714, fill opacity=0.2]
(axis cs:0,1004.49031955384)
--(axis cs:0,109.109680446157)
--(axis cs:1,149.971315342814)
--(axis cs:2,155.173927820838)
--(axis cs:3,136.89040868653)
--(axis cs:4,133.490003996803)
--(axis cs:5,136.248340530061)
--(axis cs:6,128.450033829515)
--(axis cs:7,111.757581925347)
--(axis cs:8,134.725240901843)
--(axis cs:9,138.80405516516)
--(axis cs:10,143.139977426665)
--(axis cs:10,160.860022573335)
--(axis cs:10,160.860022573335)
--(axis cs:9,153.59594483484)
--(axis cs:8,158.074759098157)
--(axis cs:7,207.042418074653)
--(axis cs:6,202.349966170485)
--(axis cs:5,182.951659469939)
--(axis cs:4,158.509996003197)
--(axis cs:3,156.70959131347)
--(axis cs:2,201.226072179162)
--(axis cs:1,193.228684657186)
--(axis cs:0,1004.49031955384)
--cycle;

\path [fill=steelblue31119180, fill opacity=0.2]
(axis cs:0,1004.49031955384)
--(axis cs:0,109.109680446157)
--(axis cs:1,88.0517719007585)
--(axis cs:2,132.550510257217)
--(axis cs:3,129.7750776405)
--(axis cs:4,133.052277442495)
--(axis cs:5,129)
--(axis cs:6,131.267949192431)
--(axis cs:7,130.150609846808)
--(axis cs:8,131.10455488499)
--(axis cs:9,129.876461593833)
--(axis cs:10,130.775255128608)
--(axis cs:10,133.224744871392)
--(axis cs:10,133.224744871392)
--(axis cs:9,133.723538406167)
--(axis cs:8,133.29544511501)
--(axis cs:7,134.249390153192)
--(axis cs:6,134.732050807569)
--(axis cs:5,133)
--(axis cs:4,134.147722557505)
--(axis cs:3,137.8249223595)
--(axis cs:2,137.449489742783)
--(axis cs:1,100.748228099242)
--(axis cs:0,1004.49031955384)
--cycle;

\addplot [semithick, black]
table {%
-0.5 132.2
10.5 132.2
};
\addlegendentry{\small Full data}
\addplot [semithick, darkorange25512714, mark=*, mark size=1, mark options={solid}]
table {%
0 556.8
1 171.6
2 178.2
3 146.8
4 146
5 159.6
6 165.4
7 159.4
8 146.4
9 146.2
10 152
};
\addlegendentry{\small sCSL}
\addplot [semithick, steelblue31119180, mark=*, mark size=1, mark options={solid}]
table {%
0 556.8
1 94.4
2 135
3 133.8
4 133.6
5 131
6 133
7 132.2
8 132.2
9 131.8
10 132
};
\addlegendentry{\small proxCSL}
\end{axis}

\end{tikzpicture}
    }
    
    \subfigure[{\it ember-100k}, 128 partitions, $\lambda=0.0001$.]{
        \begin{tikzpicture}[scale=0.8]

\definecolor{darkgray176}{RGB}{176,176,176}
\definecolor{darkorange25512714}{RGB}{255,127,14}
\definecolor{steelblue31119180}{RGB}{31,119,180}

\begin{axis}[
width=0.32\textwidth,
height=0.27\textwidth,
legend cell align={left},
legend style={fill opacity=0.8, draw opacity=1, text opacity=1, draw=none},
tick align=outside,
tick pos=left,
x grid style={darkgray176},
xlabel={Iterations},
xmin=-0.5, xmax=10.5,
xtick style={color=black},
y grid style={darkgray176},
ylabel={Objective error (\%)},
ymin=-0.105865161531867, ymax=2.22300112451379,
ytick style={color=black},
major tick length=0.1,
]
\path [fill=darkorange25512714, fill opacity=0.2]
(axis cs:0,2.09036650739912)
--(axis cs:0,0.333210120666978)
--(axis cs:1,0.0190269991305073)
--(axis cs:2,0.0279581413287873)
--(axis cs:3,0.0268109950567124)
--(axis cs:4,0.00513546080619732)
--(axis cs:5,0.00465233427858078)
--(axis cs:6,0.017210322575036)
--(axis cs:7,0.0016650975368664)
--(axis cs:8,0.0123254140666453)
--(axis cs:9,0.00188613891965651)
--(axis cs:10,0.00474450683962212)
--(axis cs:10,0.0295214499139434)
--(axis cs:10,0.0295214499139434)
--(axis cs:9,0.032901136309666)
--(axis cs:8,0.0372726723878904)
--(axis cs:7,0.0162475586190996)
--(axis cs:6,0.0431330576386277)
--(axis cs:5,0.0408411358241878)
--(axis cs:4,0.0356477435391453)
--(axis cs:3,0.0424009324033548)
--(axis cs:2,0.0502342635210075)
--(axis cs:1,0.0470892165573454)
--(axis cs:0,2.09036650739912)
--cycle;

\path [fill=steelblue31119180, fill opacity=0.2]
(axis cs:0,2.11714356605717)
--(axis cs:0,0.344761969135152)
--(axis cs:1,0.0793016673773421)
--(axis cs:2,0.00545456497751852)
--(axis cs:3,0.000444820165244347)
--(axis cs:4,0.000106838929061677)
--(axis cs:5,3.95011603035702e-05)
--(axis cs:6,1.20248205815504e-05)
--(axis cs:7,4.46351986058455e-06)
--(axis cs:8,-1.25170477095639e-06)
--(axis cs:9,-4.30529970675699e-06)
--(axis cs:10,-7.60307524662291e-06)
--(axis cs:10,0.000146021547539847)
--(axis cs:10,0.000146021547539847)
--(axis cs:9,0.000148910072437754)
--(axis cs:8,0.000153589353049019)
--(axis cs:7,0.000159473441738119)
--(axis cs:6,0.000169697754775568)
--(axis cs:5,0.000219550170524549)
--(axis cs:4,0.000376465792632208)
--(axis cs:3,0.00183173839582222)
--(axis cs:2,0.0208062803804365)
--(axis cs:1,0.295152004965255)
--(axis cs:0,2.11714356605717)
--cycle;

\addplot [semithick, darkorange25512714, mark=*, mark size=1, mark options={solid}]
table {%
0 1.21178831403305
1 0.0330581078439264
2 0.0390962024248974
3 0.0346059637300336
4 0.0203916021726713
5 0.0227467350513843
6 0.0301716901068319
7 0.00895632807798301
8 0.0247990432272679
9 0.0173936376146613
10 0.0171329783767828
};
\addplot [semithick, steelblue31119180, mark=*, mark size=1, mark options={solid}]
table {%
0 1.23095276759616
1 0.187226836171299
2 0.0131304226789775
3 0.00113827928053328
4 0.000241652360846942
5 0.00012952566541406
6 9.08612876785592e-05
7 8.1968480799352e-05
8 7.61688241390312e-05
9 7.23023863654983e-05
10 6.92092361466119e-05
};
\end{axis}

\end{tikzpicture}
        \begin{tikzpicture}[scale=0.8]

\definecolor{darkgray176}{RGB}{176,176,176}
\definecolor{darkorange25512714}{RGB}{255,127,14}
\definecolor{steelblue31119180}{RGB}{31,119,180}

\begin{axis}[
width=0.32\textwidth,
height=0.27\textwidth,
legend cell align={left},
legend style={fill opacity=0.8, draw opacity=1, text opacity=1, draw=none},
tick align=outside,
tick pos=left,
x grid style={darkgray176},
xlabel={Iterations},
xmin=-0.5, xmax=10.5,
xtick style={color=black},
y grid style={darkgray176},
ylabel={Number of non-zeros},
ymin=-326.694439976692, ymax=5100,
ytick style={color=black},
major tick length=0.1,
]
\path [fill=darkorange25512714, fill opacity=0.2]
(axis cs:0,2512.12162719103)
--(axis cs:0,169.678372808972)
--(axis cs:1,2036.34033485486)
--(axis cs:2,2437.30221989156)
--(axis cs:3,2037.16104713518)
--(axis cs:4,768.648511043885)
--(axis cs:5,485.924475988209)
--(axis cs:6,1651.94975821868)
--(axis cs:7,766.731587592627)
--(axis cs:8,946.51474648415)
--(axis cs:9,873.360277957637)
--(axis cs:10,1046.51513038929)
--(axis cs:10,2513.08486961071)
--(axis cs:10,2513.08486961071)
--(axis cs:9,2850.63972204236)
--(axis cs:8,4255.08525351585)
--(axis cs:7,1795.26841240737)
--(axis cs:6,6935.65024178132)
--(axis cs:5,5246.07552401179)
--(axis cs:4,3842.95148895612)
--(axis cs:3,4097.63895286482)
--(axis cs:2,7297.89778010844)
--(axis cs:1,5164.85966514514)
--(axis cs:0,2512.12162719103)
--cycle;

\path [fill=steelblue31119180, fill opacity=0.2]
(axis cs:0,2512.12162719103)
--(axis cs:0,169.678372808972)
--(axis cs:1,36.3813800273622)
--(axis cs:2,168.926628921036)
--(axis cs:3,162.560079681591)
--(axis cs:4,156.583544840618)
--(axis cs:5,158.647284327991)
--(axis cs:6,143.72689535238)
--(axis cs:7,145.016675936228)
--(axis cs:8,142.572821275113)
--(axis cs:9,139.977719473711)
--(axis cs:10,139.001515238483)
--(axis cs:10,165.398484761517)
--(axis cs:10,165.398484761517)
--(axis cs:9,169.622280526289)
--(axis cs:8,169.827178724887)
--(axis cs:7,174.983324063772)
--(axis cs:6,173.07310464762)
--(axis cs:5,194.552715672009)
--(axis cs:4,189.416455159382)
--(axis cs:3,182.639920318409)
--(axis cs:2,203.073371078964)
--(axis cs:1,135.218619972638)
--(axis cs:0,2512.12162719103)
--cycle;

\addplot [semithick, black]
table {%
-0.5 230.2
10.5 230.2
};
\addplot [semithick, darkorange25512714, mark=*, mark size=1, mark options={solid}]
table {%
0 1340.9
1 3600.6
2 4867.6
3 3067.4
4 2305.8
5 2866
6 4293.8
7 1281
8 2600.8
9 1862
10 1779.8
};
\addplot [semithick, steelblue31119180, mark=*, mark size=1, mark options={solid}]
table {%
0 1340.9
1 85.8
2 186
3 172.6
4 173
5 176.6
6 158.4
7 160
8 156.2
9 154.8
10 152.2
};
\end{axis}

\end{tikzpicture}
    }
    \vspace*{-1.0em}
    \caption{Iterated CSL updates using a standard solver (sSCL) and our method (\newowa) quickly converge to the optimal objective value (as defined by fitting on the full data) when the solution is sparse. However, sCSL often fails to reach the correct level of sparsity of a full data fit. Meanwhile, our specialized solver used in \newowa attains the optimal sparsity.}
    \label{fig:convergence}
\end{figure}
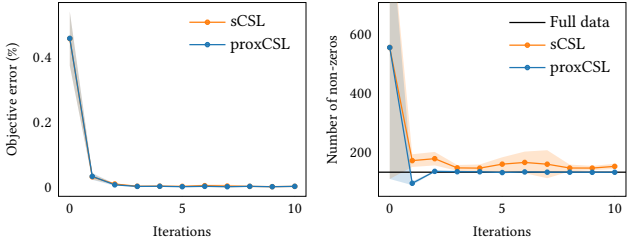
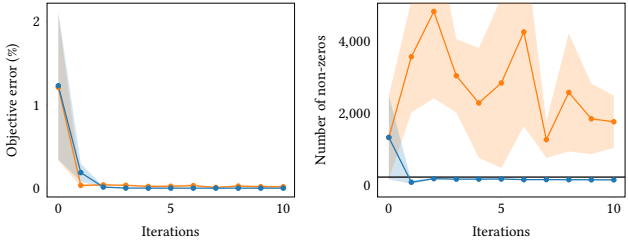

\begin{figure}[t]
    \begin{tikzpicture}[scale=0.8]

\definecolor{darkgray176}{RGB}{176,176,176}
\definecolor{darkorange25512714}{RGB}{255,127,14}
\definecolor{lightgray204}{RGB}{204,204,204}
\definecolor{steelblue31119180}{RGB}{31,119,180}

\begin{axis}[
width=0.31\textwidth,
height=0.27\textwidth,
legend cell align={left},
legend pos=north west,
legend style={fill opacity=0.8, draw opacity=1, text opacity=1, draw=none},
log basis x={2},
tick align=outside,
tick pos=left,
x grid style={darkgray176},
xlabel={Partitions},
xmin=8, xmax=512,
xmode=log,
xtick style={color=black},
y grid style={darkgray176},
ylabel={CSL Divergence},
ymin=-10, ymax=370,
ytick style={color=black},
major tick length=0.1,
]
\path [fill=steelblue31119180, fill opacity=0.2]
(axis cs:8,0.000592000000000023)
--(axis cs:8,1.40000000000001e-05)
--(axis cs:16,0.000649999999999998)
--(axis cs:32,0.000505999999999993)
--(axis cs:64,0.000464000000000006)
--(axis cs:128,0)
--(axis cs:256,0)
--(axis cs:512,0.0407558)
--(axis cs:512,0.2450392)
--(axis cs:512,0.2450392)
--(axis cs:256,0.577937)
--(axis cs:128,0.00362700000000001)
--(axis cs:64,0.0029)
--(axis cs:32,0.00158199999999999)
--(axis cs:16,0.00100800000000002)
--(axis cs:8,0.000592000000000023)
--cycle;

\path [fill=darkorange25512714, fill opacity=0.2]
(axis cs:8,0.410668)
--(axis cs:8,0)
--(axis cs:16,0.18927523)
--(axis cs:32,0)
--(axis cs:64,0.27836625)
--(axis cs:128,1.6100783)
--(axis cs:256,0)
--(axis cs:512,34.9148300000001)
--(axis cs:512,644.42977)
--(axis cs:512,644.42977)
--(axis cs:256,1264.472945)
--(axis cs:128,3.2891077)
--(axis cs:64,1.45374655)
--(axis cs:32,1.1498604)
--(axis cs:16,0.84757483)
--(axis cs:8,0.410668)
--cycle;

\addplot [semithick, darkorange25512714, mark=*, mark size=1, mark options={solid}]
table {%
8 0.1735338
16 0.51842503
32 0.5711898
64 0.8660564
128 2.449593
256 22.72744
512 339.6723
};
\addlegendentry{\small \newowa, $\alpha=0$}

\addplot [semithick, steelblue31119180, mark=*, mark size=1, mark options={solid}]
table {%
8 0.000303000000000012
16 0.00082900000000001
32 0.00104399999999999
64 0.001682
128 0.001322
256 0.008615
512 0.1428975
};
\addlegendentry{\small $\alpha=0.001$}

\end{axis}

\end{tikzpicture} \begin{tikzpicture}[scale=0.8]

\definecolor{darkgray176}{RGB}{176,176,176}
\definecolor{darkorange25512714}{RGB}{255,127,14}
\definecolor{lightgray204}{RGB}{204,204,204}
\definecolor{steelblue31119180}{RGB}{31,119,180}

\begin{axis}[
width=0.31\textwidth,
height=0.27\textwidth,
legend cell align={left},
legend style={fill opacity=0.8, draw opacity=1, text opacity=1, draw=lightgray204},
log basis x={10},
tick align=outside,
tick pos=left,
x grid style={darkgray176},
xlabel={Number of non-zeros},
xmin=100, xmax=250000,
xmode=log,
xtick style={color=black},
minor tick style={draw=none},
y grid style={darkgray176},
ylabel={CSL Divergence},
ymin=-400, ymax=5000,
ytick style={color=black},
major tick length=0.1,
]
\path [fill=steelblue31119180, fill opacity=0.2]
(axis cs:2.2,0.00706899967027164)
--(axis cs:2.2,0)
--(axis cs:10.6,0)
--(axis cs:21.8,0.00193977307697346)
--(axis cs:46.2,0.00681904680645467)
--(axis cs:114,0.00717181772110004)
--(axis cs:307,0.00596442515254595)
--(axis cs:2167,0)
--(axis cs:2983.4,0.000734186922829562)
--(axis cs:11250.2,0.00396807380295278)
--(axis cs:33308.8,0.00779418970249062)
--(axis cs:81675.4,0)
--(axis cs:156506.6,0)
--(axis cs:196935.2,0)
--(axis cs:225886.6,0.0258369852270302)
--(axis cs:232980,0)
--(axis cs:235873.4,0.0207554686548151)
--(axis cs:236702.2,0.117867314345616)
--(axis cs:236860,0.148241369418296)
--(axis cs:236894.4,0.249730194295086)
--(axis cs:236903,0)
--(axis cs:236903,0.63033606734143)
--(axis cs:236903,0.63033606734143)
--(axis cs:236894.4,1.09593512570491)
--(axis cs:236860,0.886230110581704)
--(axis cs:236702.2,0.982696205654384)
--(axis cs:235873.4,0.932016371345185)
--(axis cs:232980,0.716324636548476)
--(axis cs:225886.6,0.93614593477297)
--(axis cs:196935.2,0.652836986089475)
--(axis cs:156506.6,0.789024823191863)
--(axis cs:81675.4,0.417908489958952)
--(axis cs:33308.8,0.0175757302975094)
--(axis cs:11250.2,0.00808568619704721)
--(axis cs:2983.4,0.00377461307717043)
--(axis cs:2167,0.38164773412978)
--(axis cs:307,0.00845957484745405)
--(axis cs:114,0.0154861822789)
--(axis cs:46.2,0.0185857531935453)
--(axis cs:21.8,0.0163230269230265)
--(axis cs:10.6,0.0975749780692671)
--(axis cs:2.2,0.00706899967027164)
--cycle;

\path [fill=darkorange25512714, fill opacity=0.2]
(axis cs:2.2,0.00710663487982135)
--(axis cs:2.2,0)
--(axis cs:11,8.27608466569154e-05)
--(axis cs:22.2,0.000671206372289339)
--(axis cs:45.8,0)
--(axis cs:105.6,0.00191977707921834)
--(axis cs:235,0.00437307920812991)
--(axis cs:4422.6,0)
--(axis cs:16354.2,1.8268146363177)
--(axis cs:37128.4,8.91898855210643)
--(axis cs:56867.2,52.3559427578598)
--(axis cs:86796.6,0)
--(axis cs:149095.2,1024.52384124223)
--(axis cs:186214,627.531931856409)
--(axis cs:220547.4,1587.01325309902)
--(axis cs:221468.8,300.472587600618)
--(axis cs:225363.4,826.658094689013)
--(axis cs:227893.4,0)
--(axis cs:232616.6,2161.70389943609)
--(axis cs:235949.2,2598.9636580804)
--(axis cs:236296,3322.6408449515)
--(axis cs:236296,10103.7955550485)
--(axis cs:236296,10103.7955550485)
--(axis cs:235949.2,9008.08114191959)
--(axis cs:232616.6,9103.54490056391)
--(axis cs:227893.4,6116.83458869631)
--(axis cs:225363.4,8982.93630531099)
--(axis cs:221468.8,6873.46621239938)
--(axis cs:220547.4,8621.40634690098)
--(axis cs:186214,6721.43606814359)
--(axis cs:149095.2,7204.77375875777)
--(axis cs:86796.6,4323.05995607066)
--(axis cs:56867.2,241.29514924214)
--(axis cs:37128.4,17.0414230478936)
--(axis cs:16354.2,4.3234252036823)
--(axis cs:4422.6,230.32326539531)
--(axis cs:235,0.00786052079187007)
--(axis cs:105.6,0.00482982292078166)
--(axis cs:45.8,0.00508093142152021)
--(axis cs:22.2,0.00428999362771064)
--(axis cs:11,0.00502643915334312)
--(axis cs:2.2,0.00710663487982135)
--cycle;

\addplot [semithick, steelblue31119180, mark=*, mark size=1, mark options={solid}]
table {%
2.2 0.003278
10.6 0.0463218
21.8 0.0091314
46.2 0.0127024
114 0.011329
307 0.007212
2167 0.120386
2983.4 0.0022544
11250.2 0.00602688
33308.8 0.01268496
81675.4 0.161179074
156506.6 0.38994736
196935.2 0.31509628
225886.6 0.48099146
232980 0.31427768
235873.4 0.47638592
236702.2 0.55028176
236860 0.51723574
236894.4 0.67283266
236903 0.24296716
};
\addplot [semithick, darkorange25512714, mark=*, mark size=1, mark options={solid}]
table {%
2.2 0.00343199999999999
11 0.00255460000000002
22.2 0.00248059999999999
45.8 0.0024784
105.6 0.0033748
235 0.00611679999999999
4422.6 72.38985406
16354.2 3.07511992
37128.4 12.9802058
56867.2 146.825546
86796.6 1942.03406
149095.2 4114.6488
186214 3674.484
220547.4 5104.2098
221468.8 3586.9694
225363.4 4904.7972
227893.4 2539.5632
232616.6 5632.6244
235949.2 5803.5224
236296 6713.2182
};
\end{axis}

\end{tikzpicture}
    \vspace*{-1.0em}
    \caption{Divergence between CSL and true objective values after one \newowa update step, as a function of number of partitions (left) and intermediate solution sparsity (right). The divergence increases with decreasing sample size and increasing dimensionality, as expected. Setting the proximal parameter $\alpha > 0$ fixes the issue.}
    \label{fig:divergence}
\end{figure}
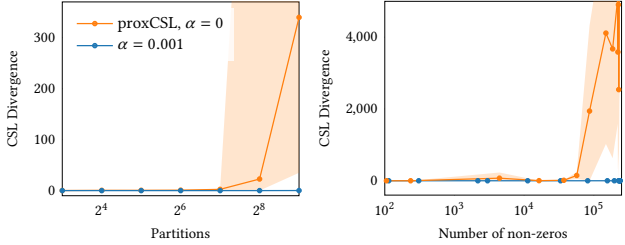

\textbf{Baselines.} For the remainder of our work we refer to two main CSL-type baselines for distributed updates: \textit{sCSL} and \textit{sDANE}.
These are sparse modifications of CSL and DANE with proximal regularization (Eq.~\ref{eq:full_csl}) which we implemented in OWL-QN.
These methods were presented and named CEASE in recent work \cite{fan2023communication}, but we use the above names to more clearly differentiate them and relate them to their predecessors.

While other global update methods have been recently proposed such as DiSCO \cite{zhang2018communication} and GIANT \cite{wang2018giant}, which solve approximate Newton problems with conjugate gradient, they do not produce sparse models so we do not compare against them.

\section{A proximal solver for sparse CSL} \label{sec:method}

In this section we will describe our algorithm \newowa, which solves the full CSL objective (\ref{eq:full_csl}) using iterative proximal Newton steps. 
\newowa converges to the global objective and true sparsity and automatically strengthens regularization when the CSL objective diverges. Our algorithms are summarized in Algo.~\ref{alg:newowa} and Algo.~\ref{alg:prox}.

\textbf{Proximal Newton.} Proximal Newton algorithms combine second-order updates with a proximal operator to handle the $L_1$ penalty. 
For composite objectives of the form $\min_w f(w) \coloneqq g(w) +  \lVert w \rVert_1$ where $g$ is convex and smooth,
the algorithm approximates $g(w)$ with a quadratic model $q_g(w)$ and iteratively minimizes $q_g(w) + \lVert w \rVert_1$.

This minimization is associated with a \textit{proximal operator}, namely $\mathrm{prox}_H(w) = \argmin_z \frac{1}{2} \lVert w - z\rVert_H^2 + \lVert w \rVert_1$, and has a closed form solution. For example, if $q_g$ could be solved with a quasi-Newton step this gives the result 
\begin{equation} w^{(t+1)} = \mathrm{prox}_H \Big( w^{(t)} - H^{-1}\nabla g(w^{(t)}) \Big). \end{equation}
Because of large $d$, explicitly computing $H$, let alone inverting it, is too costly.
Instead we solve the proximal minimization using coordinate descent, optimizing over one element of $w$ at a time. This strategy has been shown to be highly efficient for large problems in \cite{yuan2011improved,friedman2010regularization} and is used in LIBLINEAR \cite{fan2008liblinear}. See \cite{lee2014proximal} for detailed coverage and theory on proximal Newton algorithms.

For clarity, each step of the resulting algorithm first forms a quadratic approximation $q_g$ which we refer to as an \textit{outer step}.
This specific approximation is then solved using \textit{inner steps} of coordinate descent,
with each inner step involving a single pass over all the features. These steps are iterated until convergence or until an iteration limit is reached.
Following the techniques used in \cite{yuan2011improved},
our implementation avoids ever explicitly forming $H$ and achieves an inner step complexity of $\mathcal{O}(nd)$.

\textbf{Outer and inner steps.} In the context of CSL, the above procedure solves a single update $t$ of CSL, which itself can be iterated. Starting with initial estimate $\hat{w}^{(t)}$, each outer step $s$ then forms the CSL quadratic approximation
\begin{multline} \label{eq:quad_csl}
q_{\tilde{\mathcal{L}}}(w) =
\Big(     
    \nabla \mathcal{L}_k(\hat{w}^{(t, s)})
+ \nabla \mathcal{L}(\hat{w}^{(t)})
- \nabla \mathcal{L}_k(\hat{w}^{(t)})
\Big)^\top \delta \\
+ \frac{1}{2} \delta^\top H_{(k)}(\hat{w}^{(t)}) \delta + \frac{\alpha}{2} \lVert w - \hat{w}^{(t)}\rVert_2^2
\end{multline}
with $\delta \coloneqq w - \hat{w}^{(t,s)}$.

This minimization is over $d$ features.
The inner step then commences by minimizing $q_{\tilde{\mathcal{L}}}$ over one feature at a time.
This is a one-variable quadratic minimization so is computed exactly.
Because $L_1$ penalty is separable, the proximal step is applied simultaneously to each variable update via the following rule \cite{yuan2011improved}.
\begin{lem} \label{lem:prox_step}
    Given quadratic loss $\mathcal{L}$ and current iterate $w^{j-1}$ which has been updated to the $(j-1)$-th coordinate, suppose the $j$-th partial first and second derivatives are $G_j$ and $H_{jj}$ respectively. Then the problem
    $$ \min_z G_j z + \frac{1}{2} H_{jj} z^2 + |w^{j-1}_j + z|$$
    has solution
$$z =
\begin{cases}
 -\frac{G_j + 1}{H_{jj}} & \text{if $G_j + 1 \leq H_{jj} w_j^{j-1} $}  \\
 -\frac{G_j - 1}{H_{jj}} & \text{if $G_j - 1 \geq H_{jj} w_j^{j-1}$}  \\
 -w_j^{j-1} & \text{otherwise}
 \end{cases}
$$
\end{lem}
We point out that the CSL gradient and Hessian are computed once at the start of each outer step.
However, after each coordinate is updated, the local gradient for the next coordinate is affected by the previous update via a cross-term with the Hessian: $G_j = \nabla_j \mathcal{L}(\hat{w}^{(t)}) + (H_{(k)}(\hat{w}^{(t)}) \delta^{(cur)})_j$ where $\delta^{(cur)}$ is the current accumulated change to $\hat{w}^{(t,s)}$. This can be seen by expanding (\ref{eq:quad_csl}). 

In our experiments we set $S=10$ and $M=50$ max outer and inner steps, respectively (see Algo.~\ref{alg:prox}).

\textbf{Hessian caching}.
As the only vector that needs to be updated during inner steps is $H^{(k)}(w)\delta$, we note that this can be done without forming the Hessian.
The Hessian for logistic regression is
$H(w) = \frac{1}{n} X^\top D(w) X$
where $D(w)$ is diagonal with entries $d_{ii} = \pi_i ( 1-\pi_i)$,
$\pi_i$ being the predicted probability of sample $x_i$.
The main node stores $D$ as a length-$n$ vector.
Then the Hessian is implicitly updated by simply updating the vector $X\delta \in \mathbb{R}^n$ after each coordinate step. 
Each coordinate step $j$ adds $z$ to entry $j$ of $\delta$;
therefore, we add the $j$-th column of $H^{(k)}$ times $z$ to the $X\delta$ vector.

The diagonal of the Hessian is also cached as a length-$d$ vector for efficiency.

\begin{algorithm}[t]
\begin{algorithmic}[1] %
\STATE {\bf Input}: Partitioned data $\{\mathcal{D}_i\}_{i=1}^p$, regularization $\lambda$, $p$ processors, $k$ update iterations \\
\medskip
\STATE {\it // Initial distributed estimator (e.g. OWA).}
\FOR{$i \in [p]$ {\bf in parallel}}
  \STATE Solve Eq.~(\ref{eq:local_obj}) on $\mathcal{D}_i$ to get $\hat{w}_i$ using local optimizer 
\ENDFOR
\medskip
\STATE collect $W \gets [\hat{w}_1, \ldots, \hat{w}_p]$ on main node
\STATE solve Eq.~(\ref{eq:naive_avg}) or (\ref{eqn:cv_obj}) on main node to get $\hat{w}$ \\
\medskip

\STATE {\it // \newowa steps starting from $\hat{w}$.}
\STATE initialize $\hat{w}^{(0)} \gets \hat{w}$ \\
\FOR{$t \in [k]$}
    \FOR{$i \in [p]$ {\bf in parallel}}
        \STATE compute $\nabla \mathcal{L}_i(\hat{w}^{(t-1)})$ {\it // local gradient}
    \ENDFOR
    \medskip

    \STATE collect $\nabla \mathcal{L}(\hat{w}^{(t-1)}) \gets 1 / p (\sum_{i \in [p]} \nabla \mathcal{L}_i(\hat{w}^{(t-1)}))$ on main node
    \STATE obtain $\hat{w}^{(t)}$ using Algorithm~\ref{alg:prox}
\ENDFOR
\STATE {\bf return} $\hat{w}^{(k)}$
\end{algorithmic}
\caption{\newowa updates}
\label{alg:newowa}
\end{algorithm}

\begin{algorithm}[t]
\begin{algorithmic}[1] %
\STATE {\bf Input}: Local partition $\mathcal{D}_1$, regularization $\lambda$,  $\nabla \mathcal{L}(\hat{w}^{(t)})$ (global),  $\nabla \mathcal{L}_k(\hat{w}^{(t)})$ (local), max outer steps $S$, max inner steps $M$\\
\medskip
\FOR{$s \in [S]$ {\it // or until convergence}}
  \STATE compute local gradient $\nabla \mathcal{L}_k(\hat{w}^{(t,s)})$
  \STATE compute (implicitly) $H_{(k)}(\hat{w}^{(t,s)})$
  \STATE $\alpha \gets 0.0001$
  \WHILE{divergence check fails}
    \STATE $\alpha \gets 10\alpha$
  \ENDWHILE
  \STATE $\delta \gets 0$
  \FOR{$m \in [M]$ {\it // or until convergence}}
    \FOR{$j \in [d]$ {\it // one coordinate descent pass}}
        \STATE update $\delta_j$ with Eq.~(\ref{eq:quad_csl}) and Lemma~\ref{lem:prox_step}
    \ENDFOR
  \ENDFOR
  \STATE scale $\delta$ with linesearch with $\alpha$
\ENDFOR
\medskip
\STATE return $\hat{w}^{(t)} + \delta$
\medskip
\end{algorithmic}
\caption{Solving a single CSL update on main node}
\label{alg:prox}
\end{algorithm}

\textbf{Linesearch.} Each iteration of coordinate descent executes a pass over all $d$ features, updating the candidate update vector $\delta$ in place. We run $M=50$ iterations, unless convergence is reached earlier.

This is followed by a linesearch \cite{lee2014proximal}.
We replace the linesearch over the local objective with the full CSL objective (\ref{eq:full_csl}) which includes the step-size regularization. This helps prevent the diverging updates when using the unregularized CSL objective (\ref{eq:csl}).
We scale $\delta$ by $\{1, \beta, \beta^2, \ldots, \beta^{k_{max}}\}$.
For each $\beta^k \delta$, we evaluate the objective at point $\hat{w} + \beta^k \delta$, and we select $k$ and corresponding update vector which gives the lowest loss. 
We fix $\beta=0.5$ and $k_{max} = 20$.

\textbf{Adaptive tuning of $\alpha$.} Divergence of the CSL objective can be detected by sharp decrease (e.g. $20\%$) in the CSL objective (\ref{eq:full_csl})
but little change or even increase in the local objective (\ref{eq:local_obj}).
This is due to the affine term dominating the objective.
For our method we start $\alpha=0.0001$ and proceed.
If after 5 iterations of coordinate descent the above conditions are met, we scale $\alpha$ by 10 and restart.
This helps identify the minimal $\alpha$ at a relatively minor runtime cost.
Because $\alpha$ affects the objective itself, this check is only performed during the first outer step.

\section{Theoretical Results}

In order to establish the global convergence of \newowa,
we first start by establishing properties of the initial solution $\wowa$.

\begin{thm}[Thm. 4, \cite{izbicki2020distributed}]
\label{thm:owa_bound}
Given a dataset $\{ \X, \Y \}$,
parameters $\lambda$ and $p$,
the OWA (Optimal Weighted Average) technique of Izbicki and Shelton~\cite{izbicki2020distributed}
produces a solution $\wowa$ such that

\begin{equation}
\| \wowa - w^* \|_2 \le \O\left( \sqrt{\frac{\alpha_{\mathrm{hi}}}{\alpha_{\mathrm{lo}}} \cdot \frac{dt}{N}} \right)
\end{equation}

\noindent with probability $1 - e^{-t}$ for some $t > 0$,
where $w^*$ is the population risk minimizer.
\end{thm}

Here, $\alpha_{\mathrm{hi}}$ and $\alpha_{\mathrm{lo}}$ are the
maximum and minimum eigenvalues of
the Hessian of the loss at $w^*$.

Thus, once we have the initial solution $\wowa$,
we know that it is in the neighborhood of the true solution $w^*$
with high probability.
Once we have $\wowa$,
the next step of \newowa is to find
the proximal surrogate loss minimizer $\tilde{w}$ of Eq.~\ref{eq:full_csl}.

In our setting,
we choose to use {\it newGLMNET}~\cite{yuan2011improved},
although other optimization algorithms could also suffice
so long as they are guaranteed to converge to the exact solution of Eq.~\ref{eq:full_csl}.

We already know that {\it newGLMNET} converges to the exact solution of the logistic regression objective function (Eq.~\ref{eqn:obj}, Appendix A \cite{yuan2011improved}).
Using similar reasoning, we can establish that {\it newGLMNET} also converges for
the proximal surrogate loss minimizer.
Let $\proxL(w)$ be the proximal surrogate loss at iteration $t$ (Eq.~\ref{eq:full_csl}),
with minimizer $\tilde{w}$.

\begin{thm}
The newGLMNET optimizer,
on the proximal surrogate loss $\proxL(w)$ 
instead of the regular logistic loss (Eqn.~\ref{eqn:obj}),
produces an exact solution $\tilde{w}$.
\end{thm}

\begin{proof}
{\it newGLMNET} is an optimizer that fits in the framework of~\citet{tseng2009coordinate},
and as with the regular logistic regression convergence proof for {\it newGLMNET}
(Appendix A,~\cite{yuan2011improved}),
it suffices to ensure the conditions required by the framework are satisfied.

Firstly, convergence requires that the Hessian (or its estimate) $H^k$
is positive definite.
When used to solve the (standard) logistic regression objective,
{\it newGLMNET} uses $H^k = \nabla^2 \L(w) + \nu \I$
for some small $\nu$,
and the positive-definiteness of $H^k$ is known.
However, in our case,
the addition of $\nu \I$ is not necessary.
As we are optimizing the proximal surrogate loss (Eq.~\ref{eq:full_csl}),
we instead take
\begin{eqnarray}
H^k &=& \nabla^2 \proxL(w) \\
 &=& \nabla^2 \L_1(w) - \nabla^2 (w^T (\nabla \L_1(w^{(t - 1)}) - \nabla \L(w^{(t - 1)}))) + \nonumber \\
 && \;\;\;\;\;\;\nabla^2 \left( \frac{\alpha}{2} \| w - w^{(t - 1)} \|_2^2 \right) \\
 &=& \nabla^2 \L_1(w) + \alpha \I
\end{eqnarray}
\noindent which is positive definite
as long as $\alpha > 0$.

Secondly, the descent direction subproblem
when using the proximal surrogate loss must be exactly solved:
\begin{equation}
\label{eqn:newglmnet_quadratic_approx}
q_k(\delta) \coloneqq \nabla \proxL(w_k)^T \delta + \frac{1}{2} \delta^T H_k \delta + \| w_k + \delta \|_1 - \| w_k \|_1,
\end{equation}
\noindent where $w_k$ is the solution that the optimizer has found
after outer step $k$,
and $H_k$ is either the Hessian $\nabla^2 \L(w_k)$
or an approximation thereof.
Our goal at this step is to find $\argmin_\delta q_k(\delta)$.
For the original formulation, see Eq. (13),~\cite{yuan2011improved}.

As specified in the paper~\cite{yuan2011improved},
{\it newGLMNET} uses a constant step size cyclic coordinate descent
to solve Eqn.~\ref{eqn:newglmnet_quadratic_approx}.
But, this will give an inexact solution,
as noted by~\citet{friedman2010regularization}.
This issue can be resolved either by
pairing the coordinate descent with a line search,
or by replacing the stopping condition for the inner coordinate descent solver
with the adaptive condition proposed by~\citet{lee2014proximal}
(Eq. (2.23), adapted here to our notation):
\begin{equation}
\label{eqn:adaptive_cd_stepsize}
\| \nabla \proxL(w_k)^T + (H^k + (H^k)^T) \delta_j \| \le \eta_j \| \nabla \proxL(w_k)^T \|
\end{equation}
\noindent where $j$ is the iteration number of the inner coordinate descent solver.

When using that adaptive stopping condition,
so long as the step size $\eta$ is under some threshold,
using Theorem~3.1 of~\citet{lee2014proximal}
and the fact that the proximal surrogate loss is smooth,
we obtain that the inner coordinate descent will converge to the exact minimizer
of the quadratic approximation $q_k(\delta)$.
In addition, the rate of convergence can be shown to be q-linear
or q-superlinear if $\eta$ decays to 0~\cite{lee2014proximal}.

The final condition for overall convergence is that
the outer line search terminates in a finite number of iterations;
for this, it is sufficient to show that
\begin{equation}
\label{eqn:finite_termination_condition}
\| \nabla \proxL(w_1) - \nabla \proxL(w_2) \| \le \Lambda \| w_1 - w_2 \|
\end{equation}
\noindent for all $w_1, w_2 \in \mathbb{R}^d$.
As with the regular logistic regression objective,
we may also observe that $\proxL(w)$ is twice differentiable,
and thus
\begin{equation}
\| \nabla \proxL(w_1) - \nabla \proxL(w_2) \| \le \| \nabla^2 \proxL(w_3) \| \| w_1 - w_2 \|,
\end{equation}
\noindent where $w_3$ is anywhere between $w_1$ and $w_2$.
Next, note that
\begin{equation}
\nabla^2 \proxL(w_3) = \nabla^2 \L_1(w_3) + \alpha \I
\end{equation}
\noindent which is bounded using the reasoning of~\citet{yuan2011improved}:
\begin{equation}
\| \nabla^2 \L_1(w_3) + \alpha \I \| \le \| \X_1^T \| \| \X_1 \| + \alpha^2.
\end{equation}
Therefore, taking $\Lambda = \| \X_1^T \| \| \X_1 \| + \alpha^2$,
newGLMNET with the adaptive stopping condition for the inner coordinate descent solver
converges to the exact minimizer of the surrogate loss $\proxL(w)$.
\end{proof}

Consider now the error of the overall \newowa algorithm:
$\| \tilde{w} - w^* \|_2$.
Recall that our interest is in the high-dimensional sparse regime
where $d$ may be (much) greater than $n$,
and we also expect the solution $w^*$ to be (potentially highly) sparse.
Considering logistic regression,
$\nabla^2 \proxL = \X^T D \X$
for some diagonal matrix $D$,
and if $d > n$,
$\nabla^2 \proxL$ is not positive definite.
This means that strong convexity is not satisfied,
and typical large-sample analysis techniques,
such as those used for CEASE~\cite{fan2023communication}
cannot be used.

However, although $\proxL(\cdot)$ is not strongly convex,
it can be strongly convex {\em along certain dimensions}
and {\em in a certain region}.
Therefore, following~\citet{negahban2012unified} (and \cite{jordan2018communication}),
we impose some common assumptions for analysis of sparse algorithms.

\begin{defn}[Restricted strong convexity \cite{negahban2012unified, jordan2018communication}]
The single-partition loss $\L_1$ satisfies the {\em restricted strong convexity} condition with parameter $\mu$ if
\begin{equation}
\L_1(w^* + \delta) - \L_1(w^*) - \delta^T (\nabla \L_1(w^*)) \ge \mu \| \delta \|_2^2
\end{equation}
\noindent where $S = \operatorname{supp}(w^*)$,
$\delta \in C(S) \coloneqq \{ v : \| v_S \|_1 \le 3 \| v_{S^C} \|_1 \}$,
and $\mu > 0$.
\end{defn}

Here, $S$ represents the set of dimensions that have nonzero values in the optimal solution $w^*$.
Intuitively,
under this condition,
$\L_1$ is strongly convex in the cone $C(S)$ centered at $w^*$,
where $C(S)$ contains any vector $\delta$
so long as $\delta$'s components are concentrated enough
in directions orthogonal to $w^*$.

\begin{defn}[Restricted Lipschitz Hessian]
A function $f(x)$ has {\em restricted Lipschitz Hessian} at radius $R$
if for all $\delta \in C(S)$ such that $\| \delta \|_2 < R$,
\begin{equation}
\| (\nabla^2 f(x + \delta) - \nabla^2 f(x)) \delta \|_{\infty} \le M \| \delta \|_2^2.
\end{equation}
\end{defn}

For our analysis, we assume that
(1) the data $(\X, \Y)$ has random design (i.i.d. sub-Gaussian),
(2) elements of $\X$ are bounded,
(3) $\L_1$ is restricted strongly convex,
(4) $\L_1$ has restricted Lipschitz Hessian,
and (5) $\proxL$ has restricted Lipschitz Hessian.

\begin{thm}
Under the assumptions above,
given $s \coloneqq | \operatorname{supp}(w^*) |$,
if
\begin{multline}
\lambda \ge 2 \| \nabla \L(w^*) \|_{\infty} +
2 \| \nabla^2 \L(w^*) - \nabla^2 \L_1(w^*) \|_{\infty} \| w^{(t - 1)} - w^* \|_1 \\
+ (4M + \alpha) \| w^{(t - 1)} - w^* \|^2_2
\end{multline}
\noindent then it follows that
\begin{equation}
\| \tilde{w} - w^* \|_2 \le \frac{3 \sqrt{s} \lambda}{\sqrt{u + \alpha/2}}.
\end{equation}
\label{thm:general_bound}
\end{thm}

\begin{proof}
A similar result is derived as Theorem 3.5 in~\citet{jordan2018communication}
by using Corollary 1 of~\citet{negahban2012unified};
however, in our situation,
we must also consider the proximal penalty $(\alpha / 2) \| w - w^{(t - 1)} \|^2_2$.

Corollary 1 of~\citet{negahban2012unified}
requires that $\proxL(w)$ be restricted strongly convex.
Because $\L_1(w)$ is restricted strongly convex,
the restricted strong convexity of $\L(w)$ is established
by the proof of Theorem 3.5 in~\citet{jordan2018communication}.
By moving the proximal penalty terms (and proximal penalty gradient term)
to the right hand side,
we obtain
\begin{multline}
\proxL(w + \delta) - \proxL(w) - \delta^T (\nabla \proxL(w)) \ge \mu \| \delta \|_2^2 - \\
\;\;\left(\frac{\alpha}{2} \| w + \delta - w^{(t - 1)} \|_2^2 - \frac{\alpha}{2} \| w - w^{(t - 1)} \|_2^2 - \alpha \delta^T (w - w^{(t - 1)})\right)
\end{multline}
and it can be easily shown via distributivity that
\begin{equation}
\frac{\alpha}{2} \| w + \delta - w^{(t - 1)} \|_2^2 - \frac{\alpha}{2} \| w - w^{(t - 1)} \|_2^2 - \alpha \delta^T (w - w^{(t - 1)}) = \frac{\alpha}{2} \| \delta \|_2^2.
\end{equation}
\noindent and therefore $\proxL(w)$ is restricted strongly convex
with parameter $\mu + \alpha / 2$.
We also must show that $\proxL(w)$ has restricted Lipschitz Hessian.
\citet{jordan2018communication} show that $\proxL(w) - (\alpha / 2) \| w - w^{(t - 1)} \|_2^2$
(e.g., the proximal surrogate likelihood without the proximal term---or,
the surrogate likelihood)
has restricted Lipschitz Hessian.
The proximal penalty only adds a constant term to the Hessian:
\begin{equation}
\nabla^2((\alpha / 2) \| w - w^{(t - 1)} \|_2^2 = \alpha
\end{equation}
\noindent which does not change the result:
$\proxL(w)$ is restricted Lipschitz Hessian with parameter $M$.
From Corollary 1 of~\citet{negahban2012unified},
it is therefore true that
\begin{equation}
\| \tilde{w} - w^* \|_2 \le \frac{3 \sqrt{s} \lambda}{\sqrt{\mu + \alpha/2}}
\end{equation}
\noindent so long as $\lambda \ge 2 \| \nabla\proxL(w^*) \|_{\infty}$.
Via the triangle inequality,
\begin{multline}
\| \proxL(w^*) \|_{\infty} \le \| \L_1(w^*) - w^{*\top} (\nabla\L_1(w^{(t - 1)}) - \nabla\L(w^{(t - 1)})) \|_{\infty} \\
\;\;\;\;+ \frac{\alpha}{2} \| w^* - w^{(t - 1)} \|^2_2
\end{multline}
\noindent and using the results from~\citet{jordan2018communication}
to handle the left-hand side yields
\begin{multline}
\| \nabla\proxL(w^*) \|_{\infty} \le \| \nabla^2\L(w^*) - \nabla^2\L_1(w^*) \|_{\infty} \| w^* - w^{(t - 1)} \|_1 \\
+ \| \nabla\L(w^*) \|_{\infty} + \left(2M + \frac{\alpha}{2}\right) \| w^* - w^{(t - 1)} \|_2^2.
\end{multline}
and therefore the statement holds.
\end{proof}

Now, using the error bound for OWA (Theorem~\ref{thm:owa_bound})
plus the theorem above,
we can derive a bound for the error of OWAGS.

\begin{thm}
Under the assumptions above,
for some constants $c_1$, $c_2$, and $t$
that are independent of $n$, $k$, $d$, and $s$,
with probability $(1 - t)(1 - c_1 \operatorname{exp}(-c_2 n))$,
\begin{multline}
\| \tilde{w} - w^* \|_2 \le \\
\O\left( \sqrt{\frac{s \log d}{n}} \right) + \O\left( s \sqrt{\frac{d \log dt}{n} \frac{\alpha_{hi}}{\alpha_{lo}}} \right) + \O\left( \sqrt{s} \sqrt{\frac{dt}{n} \frac{\alpha_{hi}}{\alpha_{lo}}} \right).
\end{multline}
\end{thm}

\begin{proof}
The result is a straightforward combination of Theorem~\ref{thm:general_bound},
Theorem~\ref{thm:owa_bound},
and Theorem 3.7 from~\citet{jordan2018communication}.

Under the conditions of Theorem~\ref{thm:general_bound},
the exact statement of Theorem 3.7 of~\citet{jordan2018communication}
applies to the proximal surrogate $\proxL(w)$:
the proximal penalty applies only a constant shift to the Hessian $\nabla^2\proxL(w)$;
this does not affect any results from that result.

Next, from Theorem~\ref{thm:owa_bound},
we know that with probability $(1 - t)$,
\begin{equation}
\| \wowa - w^* \|_2 \le \O\left( \sqrt{\frac{\alpha_{hi}}{\alpha_{lo}} \frac{dt}{n}} \right)
\end{equation}
\noindent and it is trivial to establish a simple bound on the L1-norm:
\begin{align}
\| \wowa - w^* \|_1 &\le \sqrt{s}\ \| \wowa - w^* \|_2 \\
 &\le \O\left( \sqrt{\frac{\alpha_{hi}}{\alpha_{lo}} \frac{sdt}{n}} \right).
\end{align}
As we are using $w^{(0)} = \wowa$ and $w^{(1)} = \tilde{w}$,
the statement of the theorem follows
by substituting these terms into the statement of
Theorem 3.7 of~\citet{jordan2018communication}.
\end{proof}

Note that the bound does not depend on the number of partitions $p$.
Constants, including $\mu$ and $M$, have been omitted for simplicity;
but, as the problem $(\X, \Y)$ gets `easier'---e.g., more strongly convex---the
parameter $\mu$ increases and $M$ decreases,
the bound tightens
and the required $\lambda$ decreases.
This is an intuitive result.

\section{Results} \label{sec:results}

\begin{table}[t]
    \small
    \centering
    \begin{tabular}{lcccc}
        \toprule
        {\bf dataset} & $n$ & $d$ & nnz & size \\
        \midrule
        amazon7 & $1.3$M & $262$k & $133$M, $0.04$\% & $1.2$GB \\
        url & $2.3$M & $3.2$M & $277$M, $4\mathrm{e}{\scalebox{0.5}[1.0]{\( - \)}4}$  & $3.9$GB \\
        criteo & $45$M & $1$M & $1.78$B, $4\mathrm{e}{\scalebox{0.5}[1.0]{\( - \)}3}$\% & $25$GB \\
        ember-100k & $600$k & $100$k & $8.48$B, $10.6$\% & $61$GB \\
        ember-1M & $600$k & $1$M & $38.0$B, $4.7$\% & $257$GB \\
        \bottomrule
    \end{tabular}
    \vspace*{-0.8em}
    \caption{Datasets with uncompressed {\it libsvm}-format sizes.}
    \label{tab:datasets}
\end{table}

We run experiments to thoroughly assess the performance of \newowa relative to baselines:
\textbf{sCSL} and \textbf{sDANE},
the sparse variants of CSL~\cite{jordan2018communication} and DANE~\cite{shamir2014communication} respectively with the CEASE modifications discussed in \cite{fan2023communication};
and one-shot distributed estimators \textbf{Naive avg.} and \textbf{OWA}~\cite{izbicki2020distributed}.
For the smaller datasets we also run the serial algorithm \textit{newGLMNET} from LIBLINEAR~\cite{fan2008liblinear} to show how close \newowa can get to the full data solution.
We test our method on multiple high-dimensional datasets, with details in Table \ref{tab:datasets}.

\begin{figure}[t]
    \subfigure[{\it amazon7}, multi-core, 128 partitions.]{
        \input{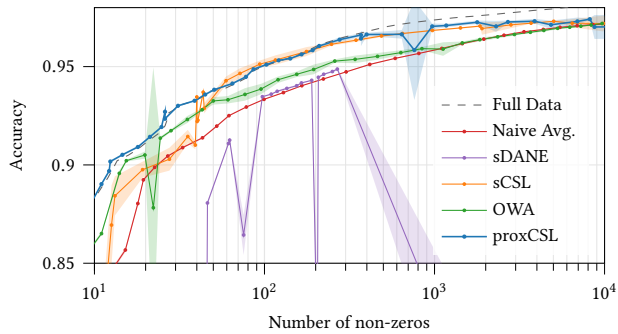}
    }
    
    \subfigure[{\it ember-100k}, multi-core, 128 partitions.]{
        \input{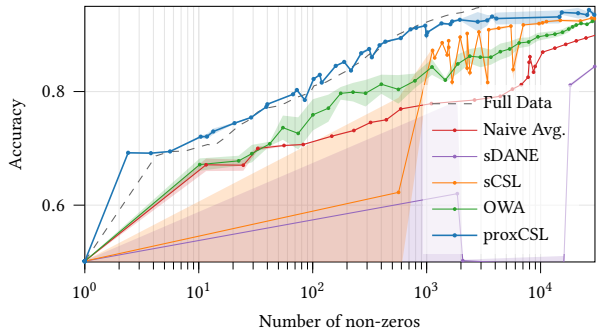}
    }
    
    \subfigure[{\it url}, multi-core, 128 partitions.]{
        \input{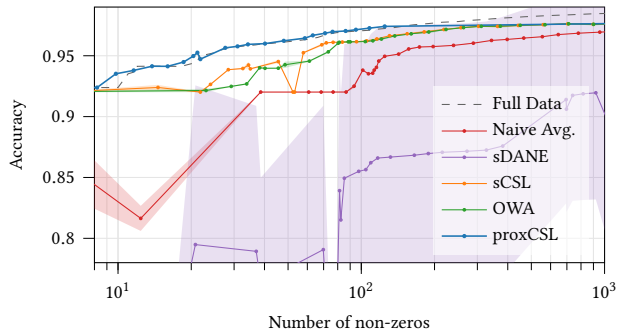}
    }
    \vspace*{-1.0em}
    \caption{Number of nonzeros vs. test set accuracy in the single-node multi-core setting over a grid of regularization values. The distributed methods (sDANE, sCSL, \newowa) are initialized with the OWA solution and updated twice. \newowa (blue) cleanly outperforms other distributed methods across the datasets, often matching the full data solution computed with LIBLINEAR (dashed grey). sCSL performs nearly as well as \newowa on \textit{amazon7} but not on other datasets. sDANE and sCSL fail to achieve sparse solutions on ember-100k even after the grid resolution was increased.}
    \label{fig:nnz-sweep-single-node}
\end{figure}

\begin{figure*}[t]
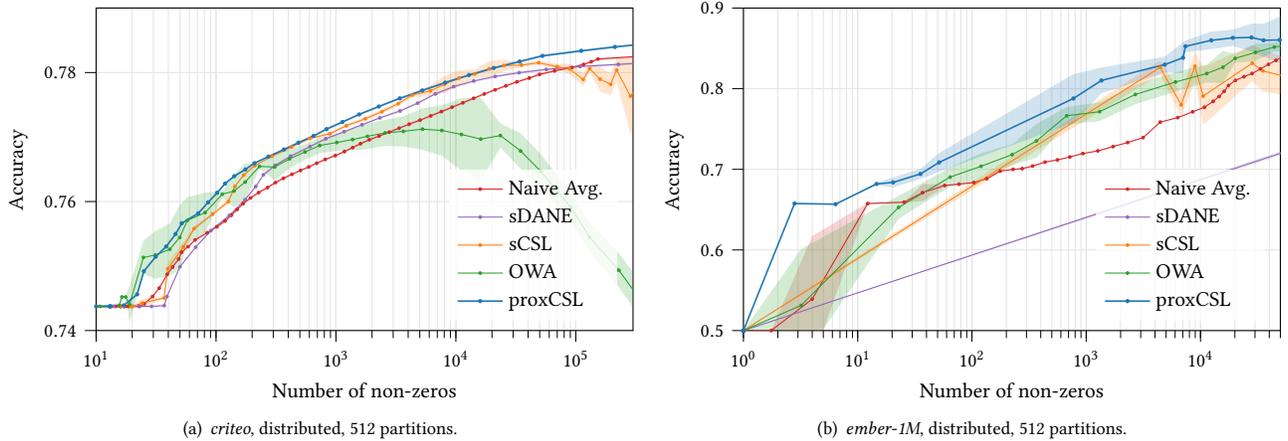

    \subfigure[{\it criteo}, distributed, 512 partitions.]{
        \input{fig/mpi_criteo_nnz_vs_test_acc.tex}
    }
    \subfigure[{\it ember-1M}, distributed, 512 partitions.]{
        \input{fig/mpi_ember1M_nnz_vs_test_acc.tex}
    }
    \vspace*{-1.0em}
    \caption{Number of nonzeros vs. test set accuracy in the distributed multi-node setting, after two update steps for the distributed methods (sDANE, sCSL, \newowa). On both datasets, \newowa (blue) outperforms the other methods across all sparsity levels. Due to the massive data size, no full data solution is computed. On criteo, OWA diverges at low regularizations, so we initialize the distributed methods with Naive Avg. instead. sDANE and sCSL fail to achieve sparse solutions on ember-1M even after the grid resolution was increased.
    }
    \label{fig:nnz-sweep-distributed}
\end{figure*}

The datasets were obtained from the LIBSVM website \cite{chang2011libsvm},
with the exception of \textit{ember-100k} and \textit{ember-1M} which we built from the malware and benign files in the Ember2018 dataset \cite{anderson2018ember} using the KiloGrams algorithm~\cite{Kilograms_2019,hashgram_2018,raff_hash_gram_parallel}.
This consists of computing 8-byte n-grams over the files in the dataset,
and subsetting to the most frequent 100k or 1M n-grams~\cite{raff_ngram_2016,Zak2017}.

For each dataset, we sample a random 80/20 train-test split.
We split the training data across varying partitions,
depending on the experiment,
and train the methods.
We repeat this process across a grid of 80 logarithmically-spaced $\lambda$ values.
For each $\lambda$, we replicate the distributed estimation 5 times and record the average number of nonzeros in each solution and average test set accuracy (along with standard deviations).
This gives a good comparison of the methods across varying sparsity levels.

We implemented all methods in C++ with the Armadillo linear algebra library \cite{sanderson2016armadillo} and mlpack machine learning library \cite{curtin2023mlpack},
with OpenMP and MPI to distribute the computations.
For sDANE and sCSL we also use the OWL-QN implementation of libLBFGS\footnote{\url{https://github.com/chokkan/liblbfgs}}
We study two distributed settings:
{\it (1)} single-node multicore, and
{\it (2)} fully distributed.
The first setting is relevant in modern servers with high numbers of cores available.
In our case, we used a powerful server with 256 cores and 4TB of RAM
for our single-node experiments.
The communication costs are lower in this setting because network latency is avoided,
but the fundamental alorithm works the same way.
The second setting is even larger in scale, when multiple machines are connected.
Here we use a cluster with 16 nodes, using up to 32 cores and 1TB of RAM on each node.

While the methods can be run for many updates, due to the goal of limiting communication,
we find that 2-4 iterations are sufficient to update the solution 
with diminishing return after (see Fig.~\ref{fig:convergence}).
Unless otherwise stated we initialize \newowa, sCSL, and sDANE with the OWA solution and compare them after 2 updates.

For additional hyperparameter and computational detail, refer to Appendix~\ref{app:detail}.

\subsection{Test accuracy across sparsity levels}
Experimental results for the single-node setting are shown over a range of $\lambda$ in Fig. \ref{fig:nnz-sweep-single-node}. 
Across a range of datasets and sparsities,
\newowa is able to converge to the full data solution after two updates.
When this occurs, \newowa often significantly outperforms the baselines sCSL and sDANE,
as well as the initial estimators OWA and the naive average.
On the smallest dataset (\textit{amazon7}) only we see that the sCSL with OWL-QN solver approaches \newowa in performance.
On the other hand, DANE generally fares significantly worse.
We find that due to averaging update models across all partitions,
sparsity is often lost.
In addition, the optimization sometimes fails to converge on one or more partitions.

Fig.~\ref{fig:nnz-sweep-distributed} shows similar experiments on the fully distributed setting, which we apply to two of our largest datasets (\textit{criteo} and \textit{ember-1M}).
As before, \newowa consistently achieves better test accuracy across most sparsity levels than the other methods. 

Our method can generalize to other loss functions provided the theoretical assumptions are met. For example, we can use the elastic net-regularized logistic regression instead of the Lasso regularization with comparable performance (Appendix~\ref{app:enet}).

\begin{table}[!t]
    \centering
    \resizebox{\columnwidth}{!}{%
    \begin{tabular}{@{}lrrrrr@{}}
    \toprule
    \multicolumn{1}{c}{\textbf{dataset}} & \multicolumn{1}{c}{Naive Avg.} & \multicolumn{1}{c}{OWA} & \multicolumn{1}{c}{sCSL} & \multicolumn{1}{c}{sDANE} & \multicolumn{1}{c}{\newowa} \\ \midrule
   \multicolumn{4}{l}{\it single-node parallel} \\
    amazon7 & 2.356s & 14.933s & 17.678s & 27.489s & 16.756s \\
    ember-100k & 7.811s & 13.245s & 48.950s & 75.120s & 70.921s \\ 
    url & 17.085s & 218.092s & 98.635 & 105.161s & 91.425s \\
    \midrule
    \bottomrule
     & & & \\
    \toprule
    \multicolumn{1}{c}{\textbf{dataset}} & \multicolumn{1}{c}{Naive Avg.} & \multicolumn{1}{c}{OWA} & \multicolumn{1}{c}{sCSL} & \multicolumn{1}{c}{sDANE} & \multicolumn{1}{c}{\newowa} \\ \midrule
    \multicolumn{4}{l}{\it fully distributed} \\
    ember-1M & 6.176s & 5.836s & 81.634s & 68.975s & 43.002s \\
    criteo & 2.349s & 25.885s & 36.069s & 65.618s & 22.293s \\
    \bottomrule
    \end{tabular}%
    }
    \vspace*{-0.6em}
    \caption{Runtime results for different techniques. Although \newowa takes longer to converge than naive averaging and OWA, it provides significantly better performance (see Fig.~\ref{fig:nnz-sweep-single-node}). This also generally holds when comparing \newowa against sCSL and sDANE.
    For more detailed timing including comparison of various steps within \newowa refer to Appendix~\ref{app:more_timing}.}
    \label{tab:runtimes}
\end{table}

\subsection{Runtime comparison}

In the next set of experiments we also compare the runtimes of our method \newowa with the other methods.
For each method we identify the setting that results in a model with roughly 1000 non-zeros for comparability and record the runtime.
Note that update methods sDANE, sCSL, and \newowa also include the initialization time in the total.
Therefore their times will generally always be higher than OWA or Naive Avg.,
unless the underlying setting was at a significantly different value of $\lambda$.

As expected, our method is quite fast even on the largest datasets.
Runtimes are comparable with sCSL since OWL-QN is also known to be a fast solver.
Yet our method converges to better accuracy solutions given similar runtime.
In comparison sDANE is generally slower because the second optimization must be done on each partition and re-averaged, incurring additional communication and computational time.
For more details and analysis see Appendix~\ref{app:more_timing}.

\subsection{Convergence to a known model}

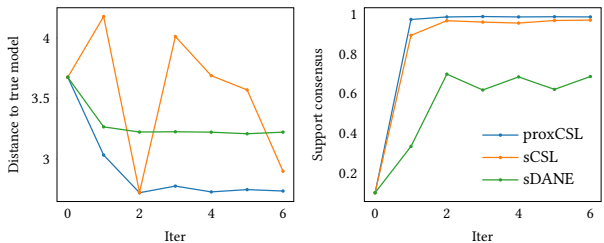
\begin{figure}[t]
\begin{tikzpicture}[scale=0.8]

\definecolor{darkgray176}{RGB}{176,176,176}
\definecolor{darkorange25512714}{RGB}{255,127,14}
\definecolor{forestgreen4416044}{RGB}{44,160,44}
\definecolor{lightgray204}{RGB}{204,204,204}
\definecolor{steelblue31119180}{RGB}{31,119,180}

\begin{axis}[
width=0.31\textwidth,
height=0.27\textwidth,
legend cell align={left},
legend style={fill opacity=0.8, draw opacity=1, text opacity=1, draw=lightgray204},
tick align=outside,
tick pos=left,
x grid style={darkgray176},
xlabel={Iter},
xmin=-0.3, xmax=6.3,
xtick style={color=black},
y grid style={darkgray176},
ylabel={Distance to true model},
ymin=2.6478495, ymax=4.2481405,
ytick style={color=black},
major tick length=0.1,
]
\addplot [semithick, steelblue31119180, mark=*, mark size=0.5, mark options={solid}]
table {%
0 3.67423
1 3.03171
2 2.72059
3 2.7753
4 2.72749
5 2.74628
6 2.73467
};
\addplot [semithick, darkorange25512714, mark=*, mark size=0.5, mark options={solid}]
table {%
0 3.67423
1 4.1754
2 2.72315
3 4.01052
4 3.68738
5 3.56958
6 2.89803
};
\addplot [semithick, forestgreen4416044, mark=*, mark size=0.5, mark options={solid}]
table {%
0 3.67423
1 3.26439
2 3.22171
3 3.22402
4 3.22061
5 3.20755
6 3.22071
};
\end{axis}

\end{tikzpicture} \begin{tikzpicture}[scale=0.8]

\definecolor{darkgray176}{RGB}{176,176,176}
\definecolor{darkorange25512714}{RGB}{255,127,14}
\definecolor{forestgreen4416044}{RGB}{44,160,44}
\definecolor{lightgray204}{RGB}{204,204,204}
\definecolor{steelblue31119180}{RGB}{31,119,180}

\begin{axis}[
width=0.31\textwidth,
height=0.27\textwidth,
legend cell align={left},
legend style={
  fill opacity=0.8,
  draw opacity=1,
  text opacity=1,
  at={(0.97,0.03)},
  anchor=south east,
  draw=none
},
tick align=outside,
tick pos=left,
x grid style={darkgray176},
xlabel={Iter},
xmin=-0.3, xmax=6.3,
xtick style={color=black},
y grid style={darkgray176},
ylabel={Support consensus},
ymin=0.0555, ymax=1.0345,
ytick style={color=black},
major tick length=0.1
]
\addplot [semithick, steelblue31119180, mark=*, mark size=0.5, mark options={solid}]
table {%
0 0.1
1 0.975
2 0.988
3 0.99
4 0.988
5 0.989
6 0.988
};
\addlegendentry{\small proxCSL}
\addplot [semithick, darkorange25512714, mark=*, mark size=0.5, mark options={solid}]
table {%
0 0.1
1 0.895
2 0.969
3 0.962
4 0.957
5 0.97
6 0.972
};
\addlegendentry{\small sCSL}
\addplot [semithick, forestgreen4416044, mark=*, mark size=0.5, mark options={solid}]
table {%
0 0.1
1 0.334
2 0.699
3 0.619
4 0.685
5 0.622
6 0.687
};
\addlegendentry{\small sDANE}
\end{axis}

\end{tikzpicture}
    \vspace*{-1.0em}
    \caption{Convergence of CSL methods to the true solution on a synthetic dataset with known generating model. \newowa outperforms the baselines in terms of model $L_2$ distance (left) as well as identifying whether a given weight should be nonzero (right).}
    \label{fig:synthetic_conv}
\end{figure}

Finally we demonstrate empirically that our method converges to the true solution on a sparse dataset with known generating model.
We simulate data where $\mathcal{X}$ has dimension $N=100000$, $d=1000$,
where each feature is mixture distribution of $U(0, 1)$ and 0 values.
The true solution $w^*$ has 100 nonzero coefficients,
and $\mathcal{Y}$ is sampled as Bernoulli from $\mathrm{expit}(X w^*)$.
Using 64 partitions the data is full-rank on each partition
in order to satisfy the strong convexity assumption.

In this data generation model, the assumptions of Thm~\ref{thm:general_bound} are satisfied so we expect convergence in the $L_2$-norm.
We train \newowa as well as baselines sSCL and sDANE with $\lambda$ set to give approximately 100 nonzeros. 
Convergence in $L_2$-norm and support recovery are shown in Fig.~\ref{fig:synthetic_conv}.
Here \newowa converges to a known solution vector faster and more accurately than the baselines.

\section{Conclusion} \label{sec:conclusion}
In this work we present \newowa which performs global updates on a distributed sparse logistic regression model in an efficient and scalable manner.
To do this, we develop a proximal Newton solver which solves a CSL-type problem effectively along with adaptive proximal regularization.
We assess our method on much larger and higher-dimension datasets than prior work,
and conclude that \newowa has much better accuracy than prior works across a wide range of model sparsities.

While we have accelerated a widely used form of logistic regression, other bespoke or customized versions still need improvement or could be integrated in the future. Coresets may be a viable approach to improving information sharing without sending all data~\cite{lu2023coreset,samadian2020unconditional} and areas like differentially privacy rely heavily on logistic regression but have far more expensive and challenging optimization problems due to the required randomness~\cite{khanna23challenge,NEURIPS2022_1add3bbd,10516654,10.1145/3605764.3623910,raff2023scaling}.

\bibliographystyle{ACM-Reference-Format}
\bibliography{refs}

\appendix

\section{Additional timing information} \label{app:more_timing}

In this section we conduct more detailed timing analysis
to breakdown the \newowa runtime
in terms of inner steps of the algorithm.
The timing breakdown helps to compare the costs of computation vs communication in our algorithms.

\begin{table}[h]
    \small
    \centering
    \begin{tabular}{lcccc}
        \toprule
        {\bf \newowa} & \multicolumn{2}{c}{ember-1M} & \multicolumn{2}{c}{criteo} \\
        \midrule
          nnz   &   1k  & 10k  & 1k &  10k \\
        \midrule
        Initial estimator & 9.47s & 7.38s & 22.8s & 21.0s \\
        Broadcast $w$ & 6.55s & 7.14s & 0.38s & 0.34s \\
        Collect grads &	3.23s & 2.24s & 1.06s &	0.95s \\
        Compute $\nabla \mathcal{L}(\hat{w})$ & 0.39s & 0.39s & 0.38s & 0.39s \\
        Full CSL update (Algo~\ref{alg:prox}) & 25.41s & 27.8s & 10.2s & 8.62s \\
        Single outer step & 2.54s & 2.78s & 1.02s & 0.86s \\
        \bottomrule
    \end{tabular}
    \caption{Detailed timing information for \newowa on two large datasets, at regularization values corresponding to 1k and 10k solution nonzeros.}
    \label{tab:timing}
\end{table}

Furthermore we compare \newowa with sCSL and sDANE on a large dataset to show the relative computation and communication times.

\begin{table}[h]
    \small
    \centering
    \begin{tabular}{lccc}
    \toprule
    criteo (10k nnz) 	& proxCSL 	& sCSL 	& sDANE \\
    \midrule
    Initial estimator 	& 21.0s 	& \textit{same} 	& \textit{same} \\
    Broadcast $w$ 	& 0.34s 	& \textit{same} 	& \textit{same} \\
    Collect grads 	& 0.95s 	& \textit{same} 	& \textit{same} \\
    Compute $\nabla \mathcal{L}(\hat{w})$ 	& 0.39s 	& \textit{same} 	& \textit{same} \\
    Full CSL update (1 node) 	& 8.62s 	& 9.26s 	& - \\
    Full CSL update (all nodes) 	& - 	& - 	& 27.1s \\
    Collect $w$'s 	& - 	& - 	& 0.95s \\
    Average final $w$ 	& - 	& - 	& 0.16s \\
    \bottomrule
    \end{tabular}
    \caption{Comparing CSL update timings of \newowa with baselines sCSL and sDANE. Since sDANE runs updates on all nodes, the update step is significantly longer.}
    \label{tab:more_timing}
\end{table}

\section{Sensitivity to hyperparameters} \label{app:detail}

Although our default \newowa sets relatively low maximum iterations,
these values are generally sufficient to ensure convergence of the objective function.
In this analysis we increase the iteration counts to guarantee full convergence and show the impact is minimal.
See following table.

\begin{table}[h]
    \centering
    \begin{tabular}{lcccc}
    \toprule
         &  OWA & $S=10,$ & $S=10,$ & $S=100,$\\
         &     &  $M=50$ & $M=1000$ & $M=1000$ \\
    \midrule
       amazon7  & 0.1111 	& 0.1057 	& 0.1057 	& 0.1057 \\
        ember100k 	& 0.3091 	& 0.2809 	& 0.2804 	& 0.2804 \\
        url 	& 0.1614 	& 0.1604 	& 0.1604 	& 0.1604 \\
    \bottomrule
    \end{tabular}
    \caption{Logistic regression objective values after running a single \newowa update with specified hyperparameters $S$ and $M$.}
    \label{tab:sensitivity}
\end{table}

We also provide additional parameter and computational details next.

\textbf{Optimizers.}
\textit{OWLQN:} We use default hyperparameters on OWL-QN for the baselines with 100 max iterations. We experimented with changing hyperparameters and increasing iterations but they did not affect the results.

\textit{LIBLINEAR:} When using LIBLINEAR to solve the initial distributed models, we set 20 max outer iterations and 50 max inner iterations. This is to obtain faster solutions since the solution is approximate anyway. This did not affect accuracy. All other parameters are default.
For the Full Data upper bound we use all default parameters.

\textbf{System details.} A single MPI experiment uses 16 machines, each of which has two AMD EPYC 7713 64-core processors. We limit to using 32 cores per machine so that the amount of inter-machine communication is non-trivial. The machines are connected via Infiniband HDR and deployed with Slurm.

\textbf{Computational complexity.} Given $S$ and $M$, as well as dataset sizes $n, D$, our solver is $O(SMnd)$ for dense data, and $O(SM z)$ for sparse, which is quite efficient. Note that here $z$ represents the number of non-zero elements in the dataset. This is the same computational complexity as newGLMNET.

\section{Elastic Net} \label{app:enet}

Our method readily extends to other sparse loss functions, including the Elastic Net-regularized logistic regression.
The following figure demonstrates \newowa with the Elastic Net penalty. 
The result is very comparable to the Lasso-regularized version.

\begin{figure}[h!]
    \centering
    \input{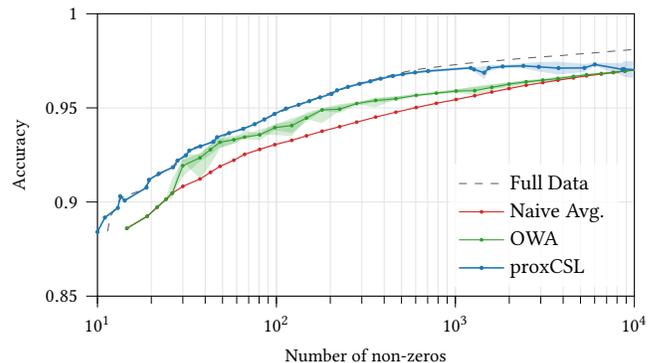}
    \caption{Number of nonzeros vs. test set accuracy for amazon7, using the Elastic Net-regularized objective.}
    \label{fig:enet}
\end{figure}

\end{document}